\documentclass[11pt]{article}

\usepackage{fullpage}
\usepackage[utf8]{inputenc} 
\usepackage[T1]{fontenc}    
\usepackage{hyperref}       
\usepackage{url}            
\usepackage{booktabs}       
\usepackage{amsfonts}       
\usepackage{nicefrac}       
\usepackage{microtype}      
\usepackage{times}
\usepackage{xcolor}
\usepackage[parfill]{parskip}
\usepackage{tikz}
\usepackage{amssymb,amsmath,graphicx,multicol,bm,graphics, bbm}
\usepackage[margin=1in]{geometry}
\usepackage[mathscr]{euscript}

\newcommand\E{\mathbb{E}}

\newcommand\R{\mathbb{R}}
\renewcommand\P{\mathbb{P}}

\newcommand{\xb}{\mathbf{x}}
\newcommand{\wb}{\mathbf{w}}

\newcommand{\VS}{\mathsf{VS}}
\newcommand\sign{\text{sign}}
\newcommand\test{\text{test}}
\newcommand{\Er}{\mathcal{E}}

\usepackage[english]{babel}
\usepackage{amsfonts}
\usepackage{amsthm}
\usepackage{mathtools,bbold}
\usepackage{graphics}
\usepackage{fancyhdr}
\usepackage{enumerate, enumitem}
\usepackage{MnSymbol}

\usepackage{caption}
\usepackage{subcaption}
\usepackage{wrapfig, blindtext}
\hypersetup{
	colorlinks,
	linkcolor={red!40!gray},
	citecolor={blue!40!gray},
	urlcolor={blue!70!gray}
}

\newtheorem{theorem}{Theorem}
\newtheorem{conjecture}{Conjecture}
\newtheorem{definition}{Definition}
\newtheorem{corollary}{Corollary}
\newtheorem{lemma}{Lemma}

\usepackage[font=small,skip=1pt]{caption}

\definecolor{darkblue}{HTML}{2f4b7c}
\definecolor{pinkred}{HTML}{f95d6a}

\title{Good Classifiers are Abundant in \\the Interpolating Regime}

\author{%
  Ryan Theisen \\
  Department of Statistics\\
  University of California, Berkeley\\
  \texttt{theisen@berkeley.edu} \\
  \and 
  Jason M. Klusowski\\
Department of Operations Research \& \\ Financial Engineering\\
Princeton University\\
\texttt{jason.klusowski@princeton.edu}
  \and 
  Michael W. Mahoney\\
ICSI and Department of Statistics\\
University of California, Berkeley\\
\texttt{mmahoney@stat.berkeley.edu}
}

\begin{document}

\date{}

\maketitle

\begin{abstract}
\noindent Within the machine learning community, the widely-used uniform convergence framework has been used to answer the question of how complex, over-parameterized models can generalize well to new data. This approach bounds the test error of the \emph{worst-case} model one could have fit to the data, but it has fundamental limitations. Inspired by the statistical mechanics approach to learning, we formally define and develop a methodology to compute precisely the full distribution of test errors among interpolating classifiers from several model classes. We apply our method to compute this distribution for several real and synthetic datasets, with both linear and random feature classification models. We find that test errors tend to concentrate around a small \emph{typical} value $\varepsilon^*$, which deviates substantially from the test error of the worst-case interpolating model on the same datasets, indicating that ``bad'' classifiers are extremely rare. We provide theoretical results in a simple setting in which we characterize the full asymptotic distribution of test errors, and we show that these indeed concentrate around a value $\varepsilon^*$, which we also identify exactly. We then formalize a more general conjecture supported by our empirical findings. Our results show that the usual style of analysis in statistical learning theory may not be fine-grained enough to capture the good generalization performance observed in practice, and that approaches based on the statistical mechanics of learning may offer a promising~alternative.
\end{abstract}

\section{Introduction}
The phenomenon of good generalization in highly over-parameterized models, including neural networks, has largely eluded theoretical understanding. 
Recently, however, progress has been made towards understanding over-parameterization in several simpler settings. 
Important examples include the variety of results demonstrating ``double descent'' phenomena in linear regression \cite{Belkin2019ReconcilingTrade-off, Bartlett2020BenignRegression, Hastie2019SurprisesInterpolation,derezinski2019exact} (and, in particular, how it is essentially a consequence of a transition between two different phases of learning~\cite{ZCM20_TR}), nearest neighbors models \cite{Xing2019BenefitAlgorithmsb}, and binary classification \cite{Chatterji2020Finite-sampleRegime, 9053524}. 
These results are typically derived by defining a specific estimator (e.g., the least-norm estimator in linear regression), and carefully examining its test risk. 
This approach presents a challenge when extending these analyses to the setting of neural networks, where no such estimator can easily be defined. 
In these situations, almost all results rely, in one way or another, on the framework of \textit{uniform convergence}; that is, results which bound a quantity of the~form
\begin{align}
    \varepsilon_{\text{unif}} := \sup_{f\in\mathcal{F}} \left | \widehat{\Er}_n(f) - \Er(f) \right|  ,
    \label{eqn:uniform-convergence}
\end{align}
where $\mathcal{F}$ is a given function class, $\widehat{\Er}_n$ is the training error on a dataset of $n$ points, and $\Er$ is the population~error. 

\begin{figure}
    \centering
    \includegraphics[scale=.5]{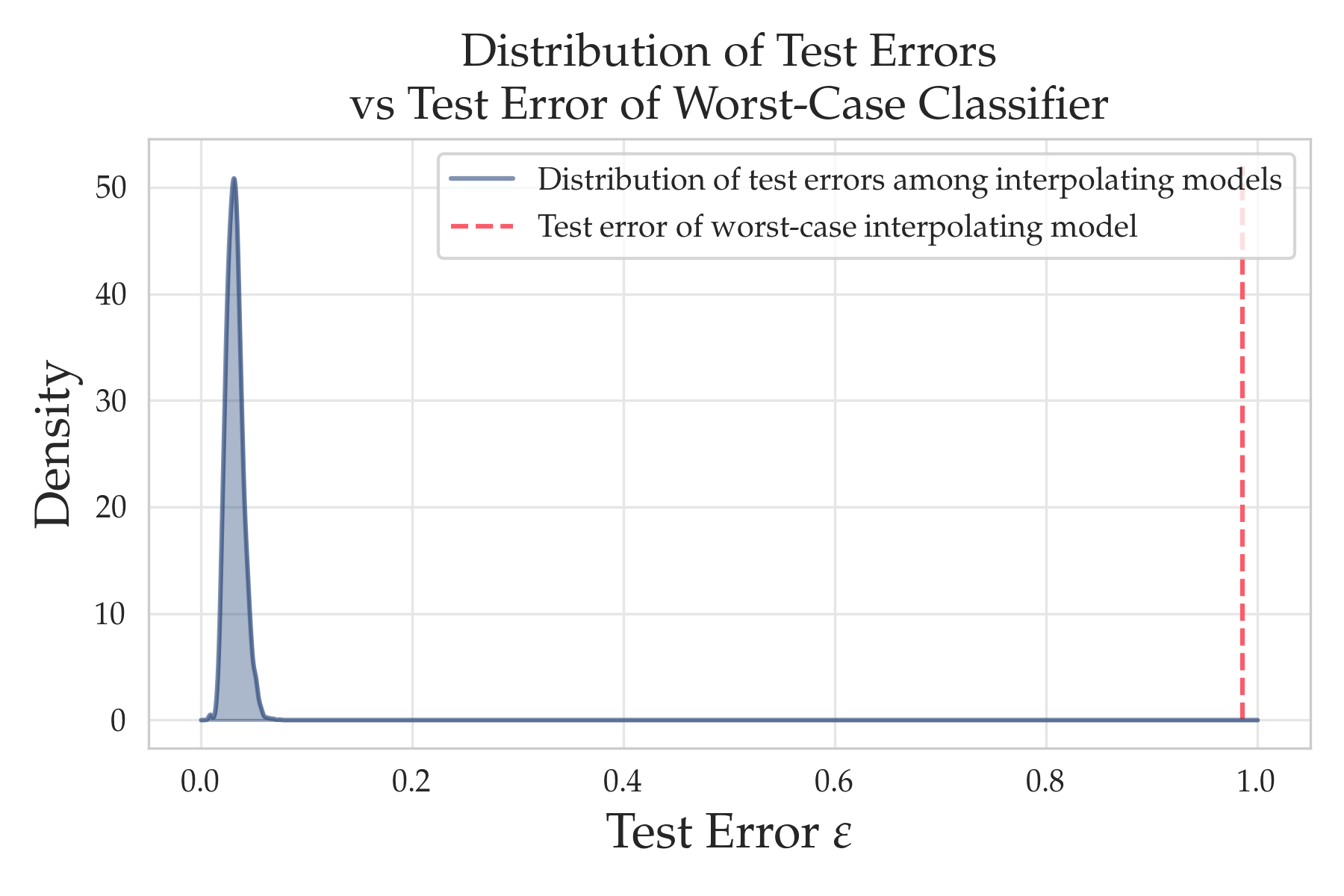}
    \caption{Test error distribution of MNIST 0 vs 1 interpolating classifiers, using $N = 1000$ random ReLU features, with $n = 500$ training samples, as well as test error of worst-case interpolating classifier. Here, for illustrative purposes, we plot the PDF (fit from a histogram using a kernel density estimate); in the remainder of the paper, we instead plot the CDFs, which can be more accurately estimated.}
    \label{fig:pdf_figure}
\end{figure}

Recently, it has been drawn into question whether this approach is fine-grained enough to capture the good generalization properties observed in deep learning \cite{MM17_TR,Nagarajan2019UniformLearning}. 
One issue that arises when using the uniform convergence framework is that for any given training set $\{(\xb_1,y_1),\dots,(\xb_n,y_n)\}$, and a sufficiently complex function class $\mathcal{F}$, the worst-case estimator $f\in\mathcal{F}$ fitting the training data may indeed perform quite poorly---thus dooming quantities like (\ref{eqn:uniform-convergence})---even if we are extremely unlikely to encounter such models in practice. 
One line of work has attempted to tackle this problem by studying the implicit biases of the algorithms used to train modern machine learning models \cite{Gunasekar2018ImplicitFactorization, Ma2020ImplicitDeconvolution, Soudry2018TheData} (by using what may be called implicit regularization in non-exact approximation algorithms~\cite{Mah12}). 
Still, such results are mostly limited to simplified settings, and a comprehensive understanding of the relationship between optimization and generalization remains~elusive.

In another line of work \cite{Wu2017TowardsLandscapes, ChoromanskaLossSurface}, it has been observed that, at least in practice for deep networks, it is not particularly important which model we obtain at the end of training; most models tend to have roughly the same test error. 
Reconciling this phenomenon with the worst-case theory must then require one of a few things to be true: i) that most models have nearly worst-case test error; ii) that models with nearly worst-case error are very rare; or iii) that worst-case bounds are simply too loose to capture the actual worst-case error. 
In this paper, we investigate these possibilities rigorously in the setting of linear and random feature classification, and we find that worst-case models with very high test error do in fact exist, but that they are exceedingly~rare.

Our approach builds conceptually on several old ideas originating out of the statistical physics literature. 
(Such a perspective, while less common in statistical learning theory today, has a long history~\cite{MM17_TR,SST92,WRB93,Haussler1996RigorousMechanics,Engel2001StatisticalLearning}.) Rather than studying the \textit{worst-case} estimator $f\in \mathcal{F}$, the statistical mechanics approach seeks to understand the behavior of the \textit{typical} function $f$. 
This typicality can be characterized in a number of ways.
A natural measure, from the statistical physics perspective, would be the \textit{entropy} (or log density of states), which captures the number of models at any given test error value. 
Analyses of learning problems have been conducted using the entropy method in a variety of simplified settings, including the case of finite $\mathcal{F}$ as well as linear classification under various simplifying assumptions on the data \cite{Haussler1996RigorousMechanics, Opper1991CalculationNoise, Engel2001StatisticalLearning}. 
Similar approaches have also been used to demonstrate the existence of phase transitions in learning behavior in logistic regression \cite{CandesLogisticRegression} and generalized linear models \cite{Barbier5451}. 
In the deep learning literature, \cite{ChoromanskaLossSurface} used the theory of spin glasses to argue that poor local minima on the training surface are rare. 
While insightful (and often technically impressive), many of these theoretical results rely on very specific assumptions on the data generating process, and hold only in the asymptotic~regime. 

In this paper, we study the behavior of test errors on real-world datasets used in practice, in a non-asymptotic regime, and without any assumptions on the data generating process. 
To do this, in Section \ref{section:definitions}, we formally define and develop a methodology to compute precisely the full distribution of test errors among interpolating classifiers from several model classes. 
In Sections \ref{section:linear-classification} and \ref{section:random-relu-features}, we then apply this methodology to compute these distributions for several real and synthetic datasets, and for both linear and random feature classification models, respectively. 
We furthermore develop a method to estimate the worst-case test errors of these classification models on the same datasets. 
Our investigation yields the following key insights:
\begin{enumerate}[nosep]
    \item Good classifiers are abundant: an overwhelming proportion of interpolating models have very small test error, relative to the worst-case error.
    \item Test errors tend to concentrate: as the size of models grow, test errors concentrate sharply around a critical value $\varepsilon^*$.
    \item There exist worst-case classifiers that are very poor: much worse than the typical classifier.
\end{enumerate}
These findings are illustrated in Figure~\ref{fig:pdf_figure}. 

To understand these observations mathematically, in Section \ref{section:theory}, we provide theoretical results in a simple setting in which we characterize the full (asymptotic) distribution of test errors, and we show that these indeed concentrate around a value $\varepsilon^*$, which we also identify exactly. 
We then formalize a more general conjecture, supported by our empirical findings, which we hope will motivate further research. 
Finally, in Section \ref{section:conclusions}, we offer some concluding thoughts, and provide several promising directions for future work. 
%
%
Proofs and additional empirical results can be found 
%
in the appendix.
%

\section{Efficiently Computing the Distribution of Test Errors for Interpolating Classifiers}
\label{section:definitions}

\subsection{Notation and Setup}
We begin with some notation that will be used throughout the paper.

We consider the setting of binary classification, and denote a training dataset by $S_n = \{(\xb_1,y_1),\dots,(\xb_n,y_n)\}$, with samples $\xb_i \in \R^d$ and labels $y_i \in \{-1,1\}$. 
We let $\mathcal{F}$ be a class of functions $f:\R^d\rightarrow \{-1,1\}$, and we define the \textit{version space} to be the following subset of~$\mathcal{F}$:
\begin{align}
    \VS(S_n) = \{f \in \mathcal{F} : f(\xb_1) = y_1,\dots,f(\xb_n) = y_n\}.
\end{align}
That is, the version space is the set of ``interpolating'' functions, i.e., those which perfectly fit the dataset $S_n$. Note that if $ \mathcal{F} $ is a linear family, then one element of the version space is the max-margin solution.
We also use $\P$ to denote a probability measure defined over $\mathcal{F}$.
We use $S_{\test} = \{(\xb_{n+1},y_{n+1}),\dots,(\xb_{n+m},y_{n+m})\}$ to denote a set of $m$ testing points, and $\Pr_{\xb,y}$ to denote a testing distribution over the data $(\xb,y)$. Using these, we define the empirical and population testing~errors:
\begin{align}
    \Er_m(f) &= \frac{1}{m}\sum_{h=1}^m \mathbb{1}(-y_{n+h} f(\xb_{n+h}) > 0),\\ 
    \Er(f) &= \Pr_{\xb,y}(-yf(\xb) > 0).
\end{align}

With these definitions in place, we can now formally define the test error distribution of interpolating~classifiers.

\begin{definition}
\label{def:rnmeps_rneps}
Given a function class $\mathcal{F}$, a measure $\P$ over $\mathcal{F}$, and a training set $S_n$, let
\begin{align}
    R_{n,m}(\varepsilon) := \frac{\P(\{\Er_m(f) \leq \varepsilon\} \cap \VS(S_n))}{\P(\VS(S_n))},
\end{align}
and
\begin{align}
R_{n}(\varepsilon) := \frac{\P(\{ \Er(f) \leq \varepsilon\} \cap \VS(S_n))}{\P(\VS(S_n))}.
\end{align}
\end{definition}

That is, the quantities $R_{n,m}(\varepsilon)$ and $R_n(\varepsilon)$ are the cumulative distribution functions (CDFs) of the errors $\Er_m$ and $\Er$, conditioned on perfectly fitting the training data. Intuitively, these quantities measure the fraction of interpolating classifiers $f\in \VS(S_n)$ that have test error at most~$\varepsilon$.

\subsection{Efficient Estimation of $R_{n,m}$}
An advantage of our definition of $R_{n,m}(\varepsilon)$ is that it is defined only relative to fixed training and testing sets, $S_n$ and $S_{\text{test}}$.
This means that, at least in principle, $R_{n,m}(\varepsilon)$ can be computed exactly (without explicit knowledge of the training and testing distributions). 
To do this na\" ively would require computing the ratio of two (in general very small) high-dimensional volumes, which would be costly and also lead to issues with numerical instability. Instead, a natural estimator for $R_{n,m}(\varepsilon)$ can be generated as follows: sample $\hat{f}_1,\dots, \hat{f}_M \sim \P(\cdot \mid \VS(S_n))$, and compute
$$
\widehat{R}_{n,m}(\varepsilon) = \frac{1}{M}\sum_{j=1}^M \mathbb{1}(\Er_m(\hat{f}_j)\leq \varepsilon).
$$
Standard Gilvenko-Cantelli-type results can be used to guarantee that $\sup_{\varepsilon} |R_{n,m}(\varepsilon) - \widehat{R}_{n,m}(\varepsilon)| = O(\frac{1}{\sqrt{M}})$.
Hence, assuming we have the ability to sample from $\P(\cdot \mid \VS(S_n))$, the distribution $R_{n,m}(\varepsilon)$ can be estimated to arbitrary precision.

For the remainder of this section, we show how we can generate samples $\hat{f} \sim \P(\cdot\mid \VS(S_n))$ for any function class of the form $\mathcal{F}_\phi = \{f(\xb) = \sign(\wb^\top \phi(\xb)) : \wb\in \R^N\}$, where $\phi:\R^d \rightarrow \R^N$ is any mapping. In this paper, we will address the following important examples:
\begin{align*}
    &\phi(\xb) = \xb, &&(\text{linear classification}) \\
    &\phi(\xb) = \sigma(\mathbf{U}\xb). &&(\text{random features})
\end{align*}

Notice that for these classes of functions, a probability measure $\P$ over $\mathcal{F}$ is simply a distribution over $\R^N$. 
Throughout this paper, we will assume that $\P$ is the uniform distribution on the sphere $\mathbb{S}^{N-1} = \{\wb \in \R^N : \|\wb\|=1\}$. This choice is made so as to obtain results that are agnostic to the choice of optimization algorithm: since any reasonable measure on the sphere will be absolutely continuous with respect to $\P$, we do not expect our main conclusions to be qualitatively changed by choosing a different base distribution.
For the sake of computation, it will be convenient to make use of the equivalence (up to scaling) of the uniform distribution with the Gaussian distribution $\mathcal{N}(\mathbf{0},\mathbf{I})$, which is a consequence of the spherical symmetry of the~Gaussian.

Let us define the function 
\begin{align}
    \mathcal{L}_n(\wb) = \prod_{i=1}^n \mathbb{1}(y_i\wb^\top\phi(\xb_i)\geq 0),
\end{align} 
and notice that $\P(\cdot \mid \VS(S_n)) = \P(\cdot \mid \mathcal{L}_n=1)$. Therefore, we are interested in drawing samples from a linearly constrained Gaussian distribution.
Fortunately, the recent work~\cite{Gessner2020IntegralsConstraints} developed the \textsc{lin-ess} algorithm (an extension of Elliptical Slice Sampling~\cite{Murray2010EllipticalSampling}) specifically for this purpose. 
Using traditional Monte Carlo methods, this task would be computationally infeasible in high dimensions, since if we na\"ively drew samples from $\P$ and rejected those not lying in the domain $\{\mathcal{L}_n(\wb) = 1\}$, then drawing a reasonable number of samples could take an exponential amount of time. 
In contrast, \textsc{lin-ess} is able to exploit special properties of the linear constraints $y_i \wb^\top \phi(\xb_i) \geq 0$ to draw samples \textit{without rejection}. 
In particular, in our setup, \textsc{lin-ess} can be used to generate samples $\widehat{\wb}_1,\dots,\widehat{\wb}_M \sim \P(\cdot \mid \mathcal{L}_n=1)$, which we can then use to compute the estimator $\widehat{R}_{n,m}(\varepsilon)$. As is the case with most MCMC algorithms, \textsc{lin-ess} is only guaranteed to produce independent samples from the posterior $\P(\cdot\mid \mathcal{L}_n = 1)$ asymptotically; we mitigate this issue in practice by using 1,000 warm-up samples, and keeping only every 10th sample~thereafter.

\section{Linear Classification}
\label{section:linear-classification}
In this section, we compute the estimated test error distributions $\widehat{R}_{n,m}(\varepsilon)$ and $\widehat{R}_n(\varepsilon)$ on both real benchmark data as well as illustrative synthetic data, for the class $\mathcal{F}_\text{LIN} = \{f(\xb) = \sign(\wb^\top \xb) : \wb\in\R^d\}$ of linear~classifiers.


\begin{figure}
    \centering
    \includegraphics[scale=.55]{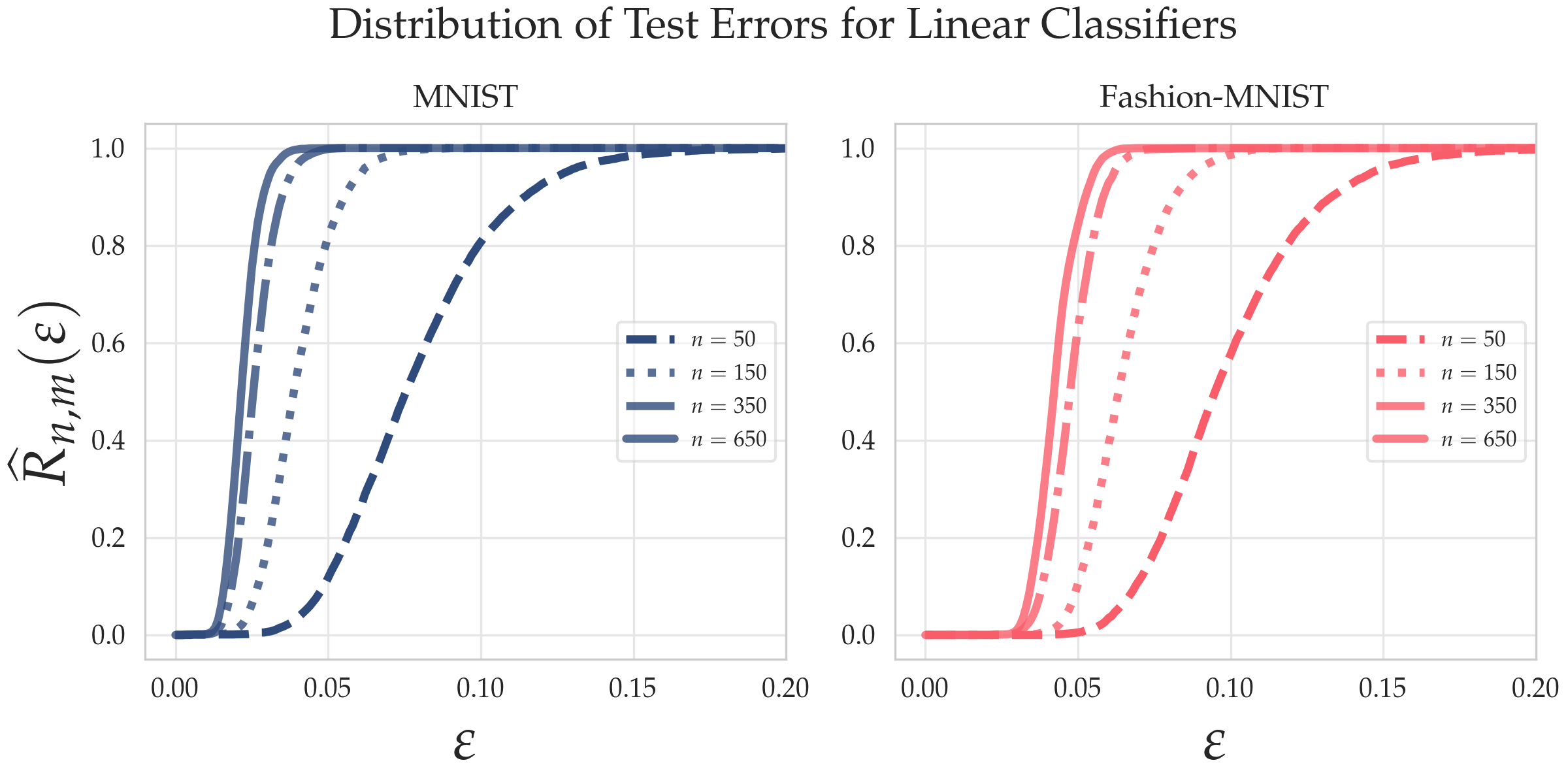} 
    \caption{Estimated test error distribution $\widehat{R}_{n,m}(\varepsilon)$ for interpolating linear classifiers on the \textsc{mnist} (0 vs 1) dataset \textbf{(\textcolor{darkblue}{blue})} and \textsc{fashion-mnist} (shirt vs pants) dataset \textbf{(\textcolor{pinkred}{red})}.}
    \label{fig:estimated-Rmn-eps-mnist}
\end{figure}

\subsection{Evaluation on Image Datasets}
\label{section:linear-image-data}
For our first set of evaluations, we compute $\widehat{R}_{n,m}(\varepsilon)$ for high-dimensional image datasets used in modern machine learning. 
In particular, we focus on the \textsc{mnist} and \textsc{fashion-mnist} datasets, which consist of images in $d=784$ dimensional space. Thus, throughout this section, we only consider values of $n<784$. 
Since we are specialized to the binary classification setting, we focus on the \textsc{mnist} 0 vs 1 task, and on the shirt vs pants task for \textsc{fashion-mnist}. For both of these tasks, the data has been centered and scaled, so as to have mean 0 and variance 1. 

In Figure \ref{fig:estimated-Rmn-eps-mnist}, we plot the $\widehat{R}_{n,m}(\varepsilon)$ for various values of $n$. For each of the plots in this section, estimators $\widehat{R}_{n,m}(\varepsilon)$ are formed with $M = 10,000$ samples from $\P(\cdot \mid \mathcal{L}_n=1)$ using the \textsc{lin-ess} algorithm, and they are evaluated on $m=5000$ testing points.

\textbf{Observation 1: Good classifiers are abundant.} 
Our first observation is that, for reasonable $n$, most interpolating classifiers have good\footnote{Of course, one could fit a model from a more complicated function class and obtain even better test performance.} test set performance. 
For example, for the \textsc{mnist} dataset, we see that at $n=350$, nearly 100\% of the models that perfectly fit the training data achieve at least 95\% ($\varepsilon = 0.05$) test accuracy. 
This indicates that, for this particular training set, bad classifiers (with error $>5\%$) make up a set with very small measure.
On the other hand, for the \textsc{fashion-mnist} task, only about 60\% of classifiers perfectly fitting the training data get 95\% test performance at $n=350$ samples, but nearly 100\% of such classifiers get 92\% accuracy.

\begin{figure}[t]
\centering
\includegraphics[scale=.55]{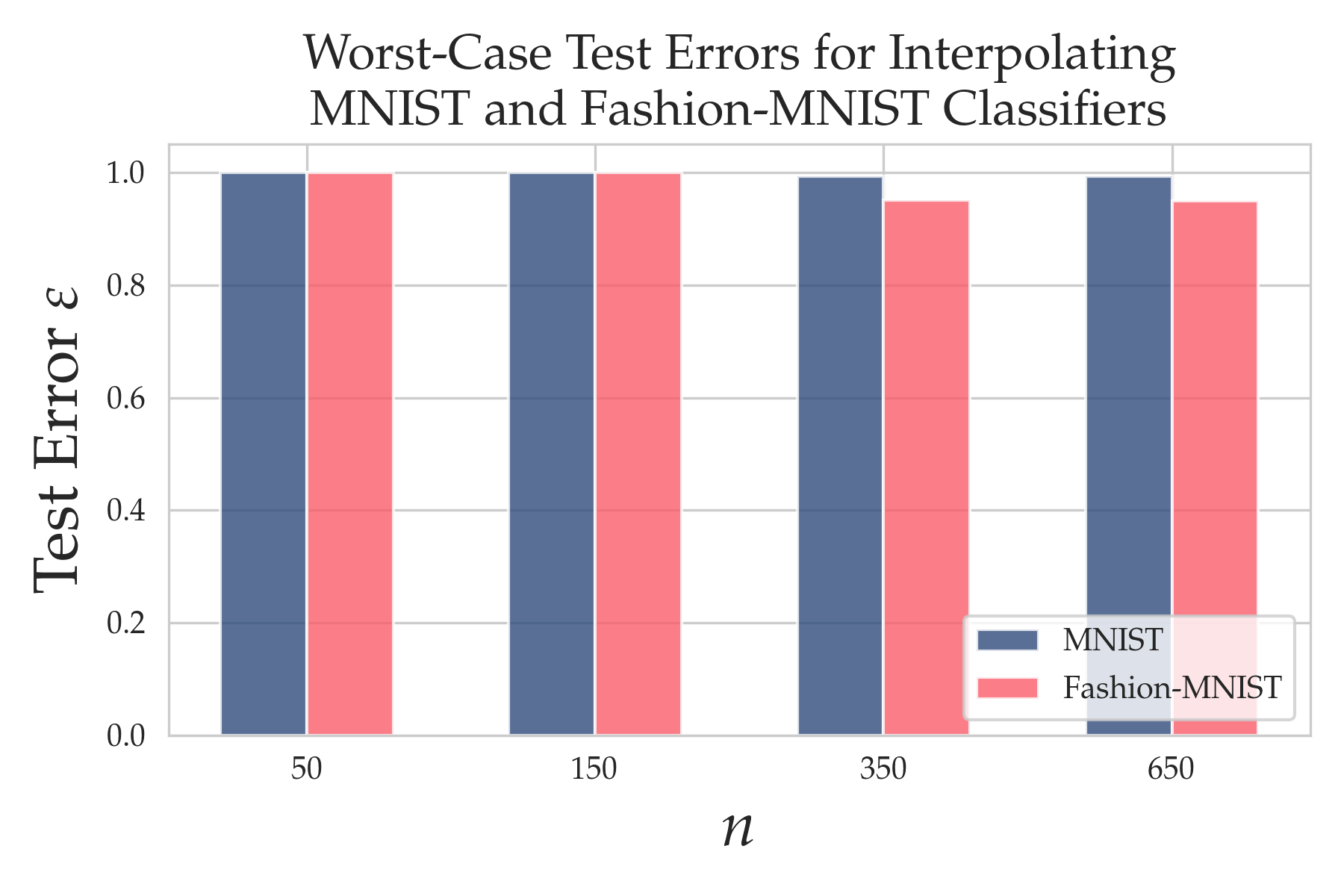}
    \caption{Test errors of interpolating classifiers with fit to $n$ ``good'' training samples and $n_b = (d-1)-n$ ``bad'' training samples. The classifiers constructed here have extremely poor test set performance, in contrast to results shown in Figure~\ref{fig:estimated-Rmn-eps-mnist}.}
    \label{fig:test-error-corrupted-labels}
\end{figure}

\textbf{Observation 2: Existence of bad classifiers.} 
A natural question that may arise out of these results is whether or not bad interpolating classifiers even exist for these tasks, at least for the parameter settings we consider. 
Here, we demonstrate a simple method for finding bad classifiers which, together with the previous results, shows that bad classifiers exist and constitute a tiny fraction of the version space. 
Given a dataset $S_n$, with $n< d$, we can append up to $n_b \leq (d-1)-n$ ``bad'' samples, to form a new dataset $S_n'$ with $n' = n+n_b$ samples. 
Notice that any model $\wb \in \VS(S_n')$ must also belong to $\VS(S_n)$, since $\VS(S_n') \subseteq \VS(S_n)$.
Here, we construct $n_b = (d-1)-n$ ``bad'' points lying in the span
of the set $\{-y_1\xb_1,\dots, -y_n\xb_n\}$. 
In Figure~\ref{fig:test-error-corrupted-labels}, we plot the test error of interpolating classifiers constructed in this manner, fit using gradient descent with a logistic loss, for varying levels of~$n$. 
We see that this method finds classifiers with test error that is nearly 1 for all values of $n$ considered.

\textit{We are therefore left with an insightful contrast:} in Figure \ref{fig:estimated-Rmn-eps-mnist}, we observe that, for example, at $n=350$, the set of interpolating \textsc{mnist} classifiers with test accuracy $\geq 95\%$ comprise a set of measure essentially 1; while in Figure \ref{fig:test-error-corrupted-labels}, we have demonstrated that there \textit{exist} interpolating classifiers for this task with test accuracy nearly $0\%$. Thus, we see that the performance of the worst-case classifier gives basically no insight into the performance of the typical classifier, indicating that a uniform convergence-type analysis is not appropriate in this setting. This is also information that cannot be gleaned by looking at a summary statistic, like the \emph{expected} test error of interpolating classifiers, i.e., $ \E[\mathcal{E}_m(\wb) \mid \VS(S_n)] $, alone—it is necessary to consider the full distribution.

\subsection{Evaluation on Synthetic Datasets}
\begin{figure*}[t] 
    \centering
    \includegraphics[scale=.5]{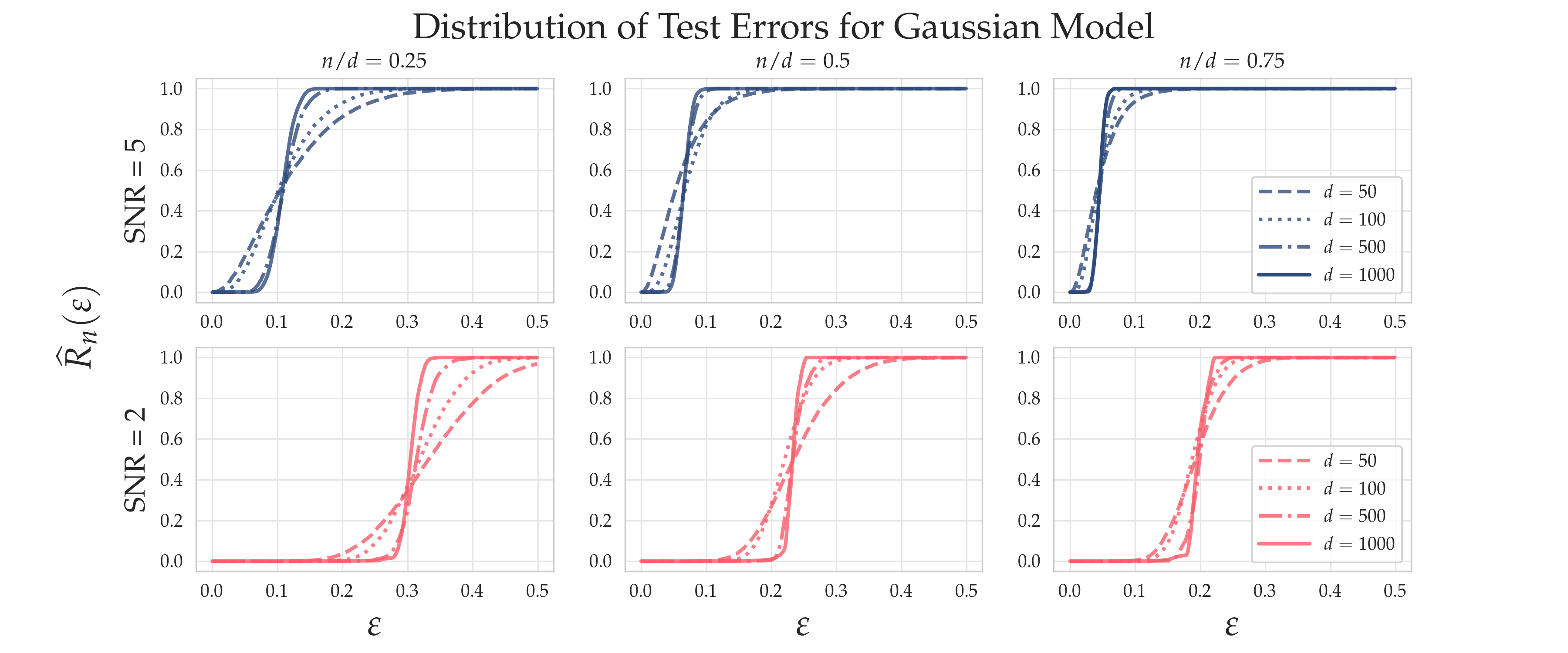}
    \caption{Plotting $\widehat{R}_n(\epsilon)$ for the Gaussian model (\ref{gaussian-data-model}) at various levels of $d$. \textbf{\textcolor{darkblue}{Blue}} curves correspond to SNR $= 5$, \textbf{\textcolor{pinkred}{red}} curves correspond to SNR $= 2$.}
    \label{fig:gaussian-data}
\end{figure*}
\label{section:gaussian-data-experiments}

For our next set of evaluations, we compute $R_n(\varepsilon)$ for synthetic data generated from the Gaussian mixture distribution
\begin{align}
(\xb, y) \sim \frac{1}{2}(N_+, 1) + \frac{1}{2}(N_-, -1). ,
\label{gaussian-data-model}
\end{align}
where $N_+ \sim \mathcal{N}(\mu, \Sigma) $, $ N_- \sim \mathcal{N}(-\mu, \Sigma)$ and $\mu \in \R^d $, $ \Sigma \in \mathcal{S}^{d}_+$. 
The purpose of this synthetic model is twofold.  First, it allows us to demonstrate the ubiquity of the phenomena observed on the \textsc{mnist} and \textsc{fashion-mnist} tasks. 
Second, it allows us to investigate the effect of varying the dimension $d$, which we could not do on the datasets studied in the previous section, as this was fixed at $d=784$. This reveals that test errors begin to concentrate around a value $\varepsilon^*$ as the dimension $d$ increases.

For this model, we have that $y\xb \sim \mathcal{N}(\mu, \Sigma)$, so we can characterize the set $\{\wb : \Er(\wb)\leq \varepsilon\}$ with the~condition
\begin{align}
\Er(\wb) \leq \varepsilon \iff \frac{\wb^\top \mu}{\sqrt{\wb^\top \Sigma \wb}} \geq -\Phi^{-1}(\varepsilon),
\end{align}
where $ \Phi(\cdot) $ is the CDF of a $ \mathcal{N}(0, 1) $ distribution.
Given a training set $S_n$ and samples $\widehat{\wb}_1,\dots, \widehat{\wb}_M \sim \P(\cdot \mid \VS(S_n))$, this expression allows us to compute an estimate $\widehat{R}_n(\varepsilon) = \frac{1}{M}\sum_{j=1}^M \mathbb{1}(\Er(\widehat{\wb}_j)\leq \varepsilon)$ in a  straightforward manner. 

As with many Gaussian models, the signal-to-noise ratio (SNR), which we define as $\sqrt{\mu^\top \Sigma^{-1}\mu}$ (or simply $\|\mu\|/\sigma$ when $\Sigma = \sigma^2I$), controls much of the complexity of this task. In Figure \ref{fig:gaussian-data}, we plot $\widehat{R}_n(\varepsilon)$ for $d = 50, 100, 500, 1000$, and with SNR $=2,5$. For these experiments, we take $\Sigma = I$ and, to keep the SNR constant as we vary the dimension, we set $\mu = (\text{SNR}/\sqrt{d},  \dots , \text{SNR}/\sqrt{d})^\top$. 

\begin{figure*}[t] 
    \centering
    \includegraphics[scale=.5]{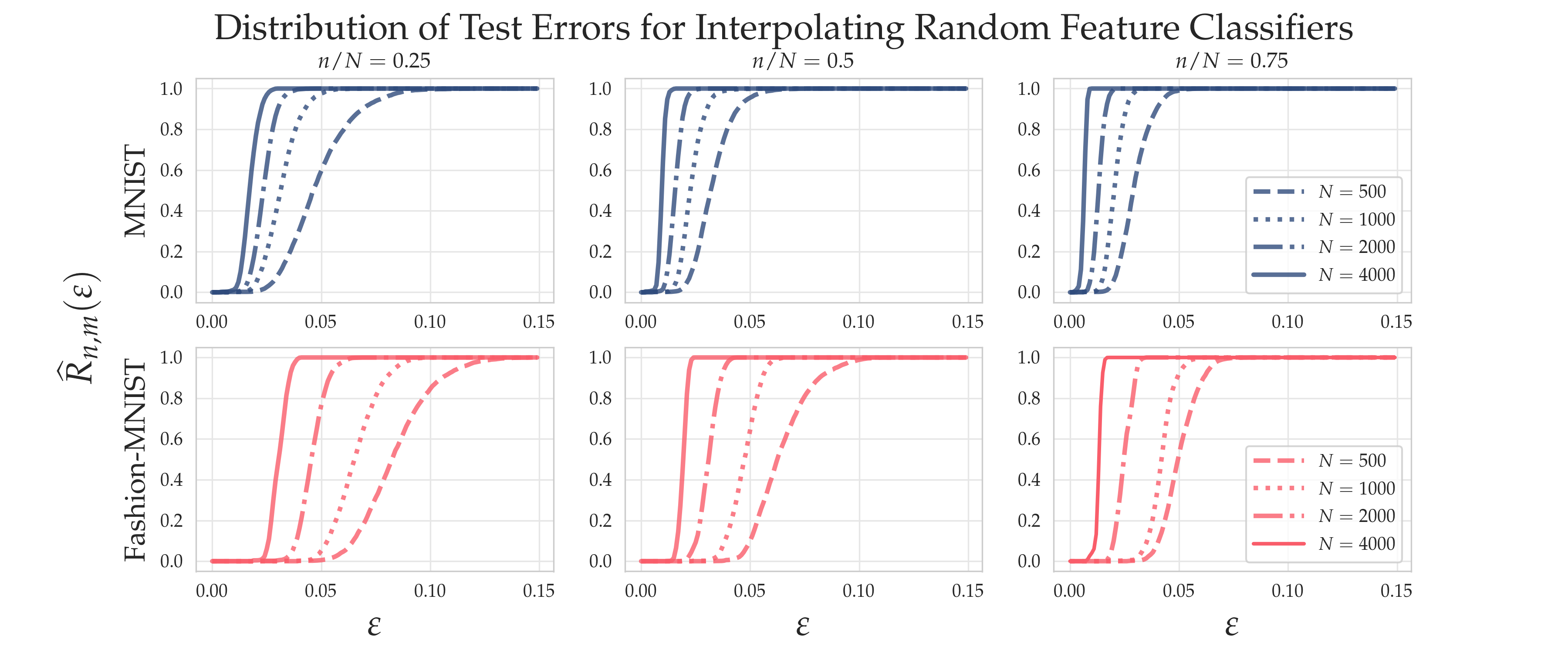}
    \caption{Plotting $\widehat{R}_{n,m}(\epsilon)$ for the random ReLU feature models on \textsc{mnist} (0 vs 1) dataset (\textbf{\textcolor{darkblue}{blue}}) and \textsc{fashion-mnist} (shirt vs pants) dataset (\textbf{\textcolor{pinkred}{red}}).}
    \label{fig:random-relu}
\end{figure*}

\textbf{Observation 3: Concentration at critical value $\varepsilon^*$.} 
Our main observation here is the existence of a critical value $\varepsilon^*$ around which test errors eventually concentrate. Indeed, we see in Figure \ref{fig:gaussian-data} that as $d$ grows, the distributions $R_n(\varepsilon)$ seem to approach the threshold function $\mathbb{1}(\varepsilon \geq \varepsilon^*)$ at a critical value $\varepsilon^*$, which depends on the aspect ratio $\alpha = n/d$. Therefore, in the large $d$ regime, almost all interpolating classifiers have test error exactly $\varepsilon^*$, and so this critical value almost completely characterizes the distribution of test errors for interpolating classifiers. We also observe that this value is largely determined by the value of the SNR. In fact, we can derive a simple lower bound on the value of $\varepsilon^*$:  
\begin{align}
\varepsilon^* \geq \Phi(-\sqrt{\mu^\top \Sigma^{-1} \mu}).
\end{align}
This corresponds to the error of the optimal Bayes classifier $\wb^\star = \Sigma^{-1}\mu$. In the next section, we observe a similar phenomenon for image classification tasks with random feature~models.

\section{Random ReLU Features}
\label{section:random-relu-features}

In this section, we consider the class of random ReLU feature classifiers $\mathcal{F}_{\text{RRF}} = \{f(\xb) = \sign(\wb^\top \phi(\xb)) : \wb\in\R^N\}$, where $\phi(\xb) = \sigma(\mathbf{U}\xb): \R^d \mapsto \R^N$. Here the rows $\mathbf{u}_1,\dots,\mathbf{u}_N$ of $\mathbf{U}$ are drawn from the uniform distribution on the sphere $\mathbb{S}^{d-1}$ and $\sigma(z) = \max(z,0)$ is the ReLU activation function. These can be viewed as one-layer ReLU networks with the weights of the first layer fixed, and they are known to enjoy universal approximation properties \cite{YitongRRF}. 

The benefit in studying such a model is that we can examine the behavior of the test error distributions as the number of hidden features $N$ grows large, with $\alpha = n/N$ fixed. This allows us to observe the critical value behavior seen in linear classification with the Gaussian model (\ref{gaussian-data-model}), but this time with the image datasets \textsc{mnist} and \textsc{fashion-mnist}.

In Figure \ref{fig:random-relu}, we plot the test error distributions for interpolating random ReLU classifiers on the \textsc{mnist} and \textsc{fashion-mnist} tasks, for various number of hidden features $N$ and ratios $\alpha = n/N$. Our main observation from these experiments is that, similar to the Gaussian model, as the number of features $N$ grows, the test errors begin to concentrate around values $\varepsilon^* \equiv \varepsilon^*(\alpha)$. Like in the Gaussian model, the critical value depends on i) the difficulty of the task (it is larger for \textsc{fashion-mnist} than for \textsc{mnist}) and ii) the aspect ratio $\alpha = n/N$. This finding indicates that the concentration phenomenon observed in Section \ref{section:gaussian-data-experiments} is quite general, and holds for both real and synthetic~datasets. 

We remark that the same technique used in Section \ref{section:linear-image-data} demonstrates that very poor classifiers also exist for the random ReLU classification models, and hence again verifies that the worst-case analysis of test errors is inappropriate for these models and datasets. 

\section{Characterizing the Distribution of Test Errors in a Simple Model}
\label{section:theory}

In this section, we present a simple model, and we prove that it exhibits the main qualitative properties we observed in Sections \ref{section:linear-classification} and \ref{section:random-relu-features}.

A full mathematical characterization of $R_{n,m}(\varepsilon)$ and/or $R_n(\varepsilon)$ is a challenging task.
To see why, let us define the random variables $\zeta_i = y_i\wb^\top \phi(\xb_i)$ for $(\xb_i, y_i)\in S_n$ and $\zeta_{n+h} = y_{n+h}\wb^\top \phi(\xb_{n+h})$ for $(\xb_{n+h}, y_{n+h})\in S_{\text{test}}$ (where we emphasize that the randomness is due to $\wb$). 
Then, for example, the normalization term $\P(\VS(S_n))$ can be expressed as
\begin{equation} \label{eq:integral}
\begin{aligned}
    \P(\VS(S_n)) &= \int \prod_{i=1}^n \mathbb{1}(y_i \wb^\top \phi(\xb_i) \geq 0)\P(\text{d}\wb)\\
    &= \P(\zeta_1\geq 0, \zeta_2\geq 0,\dots,\zeta_n\geq 0).
\end{aligned}
\end{equation}
That is, $\P(\VS(S_n))$ can be seen as an orthant probability under the distribution $\P$. 
When $\P = \mathcal{N}(\mathbf{0},\mathbf{I})$, we find that $\zeta = (\zeta_1, \zeta_2,\dots,\zeta_n) \sim \mathcal{N}(\mathbf{0}, \mathbf{A}\mathbf{A}^\top)$, where $\mathbf{A}$ is the $n\times N$ matrix whose $i^{\text{th}}$ row is $(y_i \phi(\xb_i))^\top$ and whose $ (i,j)^{\text{th}} $ entry is $ y_iy_j\phi(\xb_i)^{\top}\phi(\xb_j) $.
Computing such a Gaussian orthant probability for a general covariance matrix is a classical problem, and explicit formulae for them are known only in dimensions $\leq 5$ and in a few other special cases~\cite{Dunnett1955ApproximationsT-Distribution, Steck1962, Abrahamson1964OrthantDistribution}.

Hence, to present a model we can analyze, here we consider a simplified setting where the testing and training samples have a fixed positive correlation with each other, i.e., for fixed $ \rho \in (0, 1] $, 
\begin{equation} \label{eq:cor}
(\mathbf{A}\mathbf{A}^\top)_{ij} = y_iy_j\phi(\xb_i)^{\top}\phi(\xb_j) = \rho,
\end{equation}
for each pair of indices $i\neq j$ in $ S_n \cup S_{\text{test}} $ (where here we assume $ \phi(\xb_i) $ are normalized to have unit $ \ell^2 $ norm, without loss of generality).\footnote{By correlation between data points, we mean $ y_iy_j\phi(\xb_i)^{\top}\phi(\xb_j) $ for $i\neq j$.}
Under this assumption, we can leverage implicit expressions for the normalizing term $\P(\VS(S_n))$, which makes the problem more amenable to analysis. 


We remark that to derive asymptotically valid expressions for $ R_n(\varepsilon) $ and $ R_{n,m}(\varepsilon) $, one may be tempted to approximate \eqref{eq:integral} using off-the-shelf techniques for approximating high-dimensional integrals, e.g., Laplace's method. However, there are a number of pitfalls with this approach. First, it is difficult to quantify the approximation errors, and results that do exist are not precise enough for our purposes.
Second, certain conditions for Laplace's method or other standard integral expansions do not hold in our setting.\footnote{For example, the maximum of the function in the exponent of the integrand occurs at infinity.}
Nevertheless, we can leverage special properties of the Gaussian distribution and quantile functions to prove several non-trivial results. 
Henceforth, for sequences $ \{a_n\} $ and $ \{b_n\} $, the notation $ a_n \sim b_n $ means $ a_n = b_n(1+o(1)) $ as $ n \rightarrow \infty $.\footnote{That is, it should not be confused with ``has the probability distribution of'' which uses the same notation.}

Our first result considers the setting of a single testing point $(\xb_{n+1},y_{n+1})$, and it demonstrates 
the effect of a larger correlation $\rho$ on the probability of correctly classifying a new test point. Furthermore, it shows that, at least for this simple setting, we can expect the probability of correctly classifying a testing point to converge to 1 at a $O(1/n)$~rate.

\begin{theorem} \label{eq:supp_theorem}
Suppose we have a single testing point $(\xb_{n+1}, y_{n+1})$, which together with the training data satisfies the correlation structure \eqref{eq:cor}.
Then, as $ n\rho \rightarrow \infty $,
\begin{equation} \label{eq:supp1}
\P(y_{n+1} = \emph{sign}(\wb^\top \phi(\xb_{n+1})) \mid \mathsf{VS}(S_n)) \sim 1 - \frac{1-\rho}{n\rho}.
\end{equation}
\end{theorem}

The proof of Theorem \ref{eq:supp_theorem} relies mainly on a new asymptotic formula for the orthant probability of equicorrelated Gaussian random variables. To the best of our knowledge, this is the first of its kind, and it may be of independent interest. We state this result below in the following Lemma.

\begin{lemma}
Let $\rho \in [0,1)$ and $(X_1,...,X_n) \sim \mathcal{N}(\mathbf{0}, \boldsymbol{\Sigma})$ with $\boldsymbol{\Sigma}_{ij} = \rho$ for $i\neq j$ and $ \Sigma_{ii} = 1 $ for all $ i $. Then as $n\rho\rightarrow\infty$,
\begin{align*}
    &\P(X_1\geq 0, X_2\geq 0,\dots,X_n\geq 0) \sim 
     \sqrt{\frac{1-\rho}{\rho}} \Gamma\left(\frac{1-\rho}{\rho}\right)(4\pi\log(n))^{\frac{1}{2}(\frac{1-\rho}{\rho}-1)}n^{-\frac{1-\rho}{\rho}}.
\end{align*}
\label{thm:orthant-expression}
\end{lemma}

Theorem \ref{eq:supp_theorem} then follows by carefully evaluating the ratio of the above expression at $n+1$ and $n$.

Before stating our next result, we provide a formal definition of a critical value $\varepsilon^*$ which we will reference~therein.
\begin{definition}
We say that $ \varepsilon^* $ is a critical value if, for each $ c > 0 $, $ R_n(\varepsilon^*-c) = 0 $ and $ R_n(\varepsilon^*+c) \rightarrow 1 $ as $ n \rightarrow \infty $.
\end{definition}


Our next result provides a connection between the critical value $\varepsilon^*$, the number of training samples, and the correlation $ \rho $.

\begin{theorem} \label{eq:main_theorem}
Suppose the testing and training data satisfies the correlation structure \eqref{eq:cor}.
Let $ U $ be a gamma random variable with shape and scale parameters $ (1-\rho)/\rho $ and $ 1 $, respectively, i.e.,
$ U \sim  \text{\emph{Gamma}}(\frac{1-\rho}{\rho}, 1) $.
Then, as $ n\rho \rightarrow \infty $,
\begin{equation} \label{eq:main1}
R_n(\varepsilon) \sim \P(U\leq n\varepsilon).
\end{equation}
In particular, as $ n\rho \rightarrow \infty $,
\begin{equation} \label{eq:main2}
\varepsilon^* = \frac{1-\rho}{n\rho}
\end{equation}
is a critical value.
\end{theorem}


In this simple setting, $ n $ and $ \rho $ completely determine the distribution $R_n(\varepsilon)$: if $\rho$ is close to 1, then the data points are nearly parallel, and we will have that the test errors sharply concentrate around the critical value $\varepsilon^*$, even for $n$ small. 
Of course, in practice, there will be a more subtle and complicated relationship between the correlations and the full distribution $R_n(\varepsilon)$, which will likely be difficult to characterize precisely. 
Nonetheless, we believe that it may be possible to prove concentration in the general case, without explicitly characterizing the full distribution $R_n(\varepsilon)$. 
This is captured by the following conjecture.

\begin{conjecture}
For any model class $\mathcal{F}_{\phi}$, datasets $S_n$, testing distribution $\Pr_{\xb,y}$ (each potentially satisfying some regularity conditions) and scaling $0<\alpha <1$, there exists a critical value $\varepsilon^*(\alpha)$ such that $\lim_{n, N\rightarrow \infty, n/N \rightarrow \alpha} R_n(\varepsilon) = \mathbb{1}(\varepsilon\geq \varepsilon^*(\alpha))$ almost surely.
\end{conjecture}

Theorem \ref{eq:main_theorem} provides such a result in the case when the data is equicorrelated. Previous work using the statistical mechanics framework also prove similar results under different simplifying assumptions, namely when the features $\xb_{ik} \stackrel{\text{i.i.d.}}{\sim} \text{Unif}(\{-1,1\})$, $k=1,\dots,d$, and the labels $y_i$ are generated via a teacher model $\wb_\star$ s.t. $y_i = \sign(\wb_\star^\top \xb_i)$ (see, e.g., Chapter 2 of \cite{Engel2001StatisticalLearning}). However, these results typically only focus on the $n>d$ case, which is less relevant to the modern machine learning~regime.

\section{Discussion and Conclusion}
\label{section:conclusions}

In this paper, we built on previous literature on the statistical mechanics of learning to develop a framework to study the \textit{typical} test error of a classifier, and we propose this as an alternative to the more standard uniform convergence approach. 
We formally define the full distribution of test errors among interpolating classifiers and introduce a method to compute this distribution accurately on real datasets. 
One of the most important findings of our investigation is that, given a particular training and testing setup, there exists a critical value $\varepsilon^*$ around which almost all interpolating classifiers' test errors eventually concentrate. 
This will not come as a surprise to the statistical physicist: such typical values commonly appear in physical systems. 
However, as we have demonstrated, this critical value can differ significantly from the error $\varepsilon_{\text{unif}}$, which one would obtain via a uniform convergence analysis, especially in the interpolating/over-parameterized regime, and which may be more familiar to the machine~learner. 

Our results should motivate further research into alternatives to the uniform convergence framework, either through the lens of statistical physics or some other (likely related) perspective, and they should ultimately help resolve questions surrounding the good performance of over-parameterized machine learning models. 
As a first step, we state a few potential directions for future work building off of the results presented here. 

\textbf{More general function classes.}
While encompassing many models of interest, the function classes $\mathcal{F}_\phi$ of course do not include general neural network architectures. In this paper, we studied random feature models, which can be interpreted as neural networks with internal weights fixed at a random initialization. Another interesting setting which may be more tractable to study would be that of linearized networks of the~form
\begin{align}
    f(\xb) = \sign(\wb^\top \nabla F(\xb; \wb_0))
\end{align}
where $F$ is an arbitrary neural network with random initialization $\wb_0$. A variety of results have shown that these models coincide with neural networks in the large-width limit via the neural tangent kernel \cite{JacotNTK2018, AroraNTK2019}. While our approach would, in theory, work out-of-the-box for these models, in practice, these involve a very large number of features (approximately $O(LN^2)$, where $L$ is the number of layers, and $N$ is the width of each layer). We found that even with the \textsc{lin-ess} algorithm, sampling from $\P(\cdot \mid \VS(S_n))$ was impractical for these models. However, developing other methods for computation in this setting could yield interesting insights into the advantages (and disadvantages) of various network architectures. 

\textbf{Beyond the interpolating regime.}
The motivation for our studying interpolating classifiers comprising the version space $\VS(S_n)$ was previous work in the statistical mechanics literature, as well as the well-known worst-case results for these models given by, e.g., Vapnik–Chervonenkis theory. However, this is not the only method one could use to study the distribution of test errors. A promising alternative would be to consider the distribution over weights $\wb$ induced by some optimization algorithm, such as stochastic gradient descent (SGD). Indeed, previous work has shown that under various assumptions, SGD produces a Gaussian stationary distribution over weights $\wb$ \cite{Mandt2017StochasticInference}.  Under other (probably more realistic) assumptions, it leads to heavy-tailed structure in the weights \cite{HodMah20A_TR,GSZ20_TR}.  An intriguing direction for future work would be to study the distribution over test errors $\Er(\wb)$ induced by such a stationary distribution. It is possible that this may even simplify the theoretical investigation: whereas we studied weights drawn from $\P(\cdot\mid \VS(S_n))$ (a rather complicated distribution), it may be easier to study weights drawn from a Gaussian (or some other tractable) distribution. 

\subsubsection*{Acknowledgments}
MM would like to acknowledge DARPA, NSF, and ONR for providing partial support of this work. JK would like to acknowledge funding from NSF DMS-1915932 and NSF HDR TRIPODS DATA-INSPIRE DCCF-1934924. We also thank the authors of \cite{Gessner2020IntegralsConstraints} for sharing their implementation of the \textsc{lin-ess} algorithm. 

\bibliographystyle{alpha}
\bibliography{references}

\appendix
\section{Technical Results}
Here, we provide proofs of our main results.
\subsection{Proofs of Theorems \ref{eq:supp_theorem} and \ref{eq:main_theorem}}
\begin{proof}[Proof of Theorem \ref{eq:supp_theorem}]
The assumption \eqref{eq:cor} on the correlation structure of the data implies that 
\begin{equation} \label{eq:zeta_rep}
\zeta_i \stackrel{d}{=} \sqrt{1-\rho}Z_i+\sqrt{\rho} Z, 
\end{equation}
for $ i = 1, 2, \dots, n+m $, where $ Z, Z_1, Z_2, \dots, Z_{n+m}$ are i.i.d. $ \mathcal{N}(0, 1) $. Let $ a = \sqrt{\frac{\rho}{1-\rho}} $. According to (11), $\P(\VS(S_n))$ can be expressed as
\begin{align}
\P(\zeta_1\geq 0, \zeta_2\geq 0,\dots,\zeta_n\geq 0) & = \E_{Z\sim N(0,1)}[(\Phi(aZ))^{n}] \nonumber \\
& = \int_{-\infty}^{\infty}(\Phi(az))^{n}\phi(z)dz, \label{eq:int_main}
\end{align}
where $ \phi(\cdot) $ is the density of a $\mathcal{N}(0, 1)$ distribution.
Make the change of variables $ u = n(1-\Phi(az)) $. Then the integral \eqref{eq:int_main} becomes
\begin{equation} \label{eq:int_part}
\frac{(2\pi)^{\frac{1}{2}(\frac{1}{a^2}-1)}}{na}\int_0^{n}(1-u/n)^{n}(\phi(\Phi^{-1}(1-u/n)))^{\frac{1}{a^2}-1}du,
\end{equation}
where $ \Phi^{-1}(\cdot)$ is the quantile function of $\mathcal{N}(0, 1)$.
Next, consider the so-called ``density quantile function'' $ \phi(\Phi^{-1}(v)) $. Using a standard asymptotic expression for Mills' ratio \cite{Feller1968AnApplications}, we have
\begin{equation} \label{eq:asymp1}
\frac{1-\Phi(x)}{\phi(x)} = \frac{1}{x}\Big(1+O\Big(\frac{1}{x^2}\Big)\Big), \quad x \rightarrow \infty.
\end{equation}
Furthermore, the quantile function $ \Phi^{-1}(v) $ has the following asymptotic expression \cite{Fung2018QuantileFunctions}
\begin{equation} \label{eq:asymp2}
\Phi^{-1}(v) = \sqrt{2\log(1/(1-v))}\bigg(1+O\bigg(\frac{\log\log(1/(1-v))}{\log(1/(1-v))}\bigg)\bigg),
\end{equation}
as $ v \uparrow 1 $.
Combining these two facts (\eqref{eq:asymp1} and \eqref{eq:asymp2}) yields
\begin{align*}
& \phi(\Phi^{-1}(v)) = (1-v)\sqrt{2\log(1/(1-v))} \times \\ & \qquad\qquad\qquad\quad \bigg(1+O\bigg(\frac{\log\log(1/(1-v))}{\log(1/(1-v))}\bigg)\bigg), \quad v \uparrow 1.
\end{align*}
Using this asymptotic expression for $ \phi(\Phi^{-1}(v)) $, we find that 
$$ \int_0^{n}(1-u/n)^{n}(\phi(\Phi^{-1}(1-u/n)))^{\frac{1}{a^2}-1}du $$
from \eqref{eq:int_part} is asymptotically
\begin{equation} \label{eq:int_pen}
\begin{aligned}
& n^{1-1/a^2}(2\log(n))^{\frac{1}{2}(\frac{1}{a^2}-1)}\times \\ & \qquad \int_0^{n}\Big(1+\frac{\log(1/u)}{\log (n)}\Big)^{\frac{1}{2}(\frac{1}{a^2}-1)}u^{1/a^2-1}e^{-u}du.
\end{aligned}
\end{equation}
Finally, by the dominated convergence theorem, $ \int_0^{n}(1+\frac{\log(1/u)}{\log (n)})^{\frac{1}{2}(\frac{1}{a^2}-1)}u^{1/a^2-1}e^{-u}du $ is asymptotically
$$
\int_0^{\infty}u^{1/a^2-1}e^{-u}du = \Gamma(1/a^2).
$$
Therefore, \eqref{eq:int_pen} is asymptotically 
\begin{equation} \label{eq:express}
n^{1-1/a^2}(2\log(n))^{\frac{1}{2}(\frac{1}{a^2}-1)}\Gamma(1/a^2).
\end{equation}
Combining \eqref{eq:int_part} and \eqref{eq:express}, we have
\begin{equation} \label{eq:asymp}
\begin{aligned}
& \P(\zeta_1\geq 0, \zeta_2\geq 0,\dots,\zeta_n\geq 0) \sim \\ & \qquad \frac{\Gamma(1/a^2)}{a}(4\pi\log(n))^{\frac{1}{2}(\frac{1}{a^2}-1)}n^{-1/a^2}.
\end{aligned}
\end{equation}
Having given an asymptotic expression for the orthant probabilities, we turn our attention to the ratio $ \P(y_{n+1} = f(\xb_{n+1}) \mid \mathsf{VS}(S_n)) $, which equals
\begin{equation} \label{eq:R1}
\frac{\P(\{y_{n+1} = f(\xb_{n+1})\}\cap \VS(S_n))}{\P(\VS(S_n))}.
\end{equation}
Next, we recognize that the set $ \{y_{n+1} = f(\xb_{n+1})\}\cap \VS(S_n) $ is another version space with $ n+1 $ sample points and the same correlation structure as before, per the assumption of the theorem.
Therefore, using the above asymptotic formula \eqref{eq:asymp} for $ \P(\zeta_1\geq 0, \zeta_2\geq 0,\dots,\zeta_n\geq 0) $, we have that \eqref{eq:R1} is asymptotically
\begin{align*}
\frac{\frac{\Gamma(1/a^2)}{a}(4\pi\log(n+1))^{\frac{1}{2}(\frac{1}{a^2}-1)}(n+1)^{-1/a^2}}{\frac{\Gamma(1/a^2)}{a}(4\pi\log(n))^{\frac{1}{2}(\frac{1}{a^2}-1)}n^{-1/a^2}} = 1-\frac{1/a^2}{n}+\frac{1/a^2-1}{2n\log(n)}+O(1/n^2),
\end{align*}
thus completing the proof.
\end{proof}

\begin{proof}[Proof of Theorem \ref{eq:main_theorem}]
To begin, note that by definition,
\begin{equation} \label{eq:ratio}
R_{n,m}(\varepsilon) = \frac{\P(\{\Er_m(f) \leq \varepsilon\} \cap \VS(S_n))}{\P( \VS(S_n))}.
\end{equation}
Using the representation \eqref{eq:zeta_rep} of $ \zeta_i $ in terms of $ Z $ and $ Z_i $, we have $ \Er_m(f) \stackrel{d}{=} \frac{1}{m}\sum_{i=1}^m \mathbb{1}(Z_i < -aZ) $ and $ \Er(f) \stackrel{d}{=} \lim_m \frac{1}{m}\sum_{i=1}^m \mathbb{1}(Z_i < -aZ) $, where $ a = \sqrt{\frac{\rho}{1-\rho}} $. Henceforth, we take these distributional equivalents as the definitions of $\Er_m(f) $ and $ \Er(f) $. Now,
\begin{align*}
& \P(\{\Er_m(f) \leq \varepsilon\} \cap \VS(S_n)) = \\ & \qquad \E_{Z\sim N(0,1)}[\P(\Er_m(f)\leq\varepsilon\mid Z)\Phi^n(aZ)].
\end{align*}
By the strong law of large numbers, given $ Z $, $ \frac{1}{m}\sum_{i=1}^m \mathbb{1}(Z_i < -aZ) $ converges almost surely (with respect to the test data $ S_{\text{test}}$ and $Z_1, Z_2, \dots $) to its mean $ \Phi(-aZ) = 1-\Phi(aZ) $. Thus, by the dominated convergence theorem, almost surely,
$
\lim_m\P(\Er_m(f)\leq \varepsilon \mid Z) = \mathbb{1}(1-\Phi(aZ) \leq \varepsilon).
$
Therefore, it follows that, almost surely,
$
\P(\Er(f)\leq\varepsilon\mid Z) = \mathbb{1}(1-\Phi(aZ) \leq \varepsilon).
$
Next,
\begin{align}
R_n(\varepsilon) & = \frac{\P(\{\Er(f) \leq \varepsilon\} \cap \VS(S_n))}{\P( \VS(S_n))} \nonumber \\
& = \frac{\E_{Z\sim N(0,1)}[\mathbb{1}( 1-\Phi(aZ) \leq \varepsilon)\Phi^n(aZ)]}{\E_{Z\sim N(0,1)}[\Phi^n(aZ)]} \nonumber \\ & = 
\frac{\int_{0}^{\varepsilon n}(1-u/n)^{n}(\phi(\Phi^{-1}(1-u/n)))^{\frac{1}{a^2}-1}du}{\int_0^{n}(1-u/n)^{n}(\phi(\Phi^{-1}(1-u/n)))^{\frac{1}{a^2}-1}du}, \label{eq:int_ratio}
\end{align}
where for the final equality, we use \eqref{eq:int_part} from the proof of Theorem \ref{eq:supp_theorem}.
Using the same techniques as Theorem \ref{eq:supp_theorem} to derive asymptotic integral expressions therein (in fact, the integrands of the integrals are identical), \eqref{eq:int_ratio} is asymptotically equivalent to
$$
\frac{\int_{0}^{\varepsilon n}u^{1/a^2-1}e^{-u}du}{\int_{0}^{\infty}u^{1/a^2-1}e^{-u}du} = \P(U \leq n\varepsilon),
$$
which proves the first claim \eqref{eq:main1}. 

To prove the second claim \eqref{eq:main2} about the critical value, let $ c > 0 $ be arbitrary. Then, 
$$
\P(U \leq n(\epsilon^*+c)) \rightarrow 1,
$$
provided $ n\rho \rightarrow \infty $.
On the other hand, for $ n\rho $ large enough, $ (1-\rho)/\rho - n c < 0 $ and hence,
$$
\P(U \leq n(\epsilon^*-c)) = 0. \qedhere
$$
\end{proof}

\subsection{A new asymptotic formula for the orthant probability of equicorrelated Gaussians}
A consequence of the proof of Theorem \ref{eq:supp_theorem} is the following asymptotic expression for the orthant probability of an equicorrelated Gaussian which, to the best of our knowledge, is new. This is referred to as Lemma \ref{thm:orthant-expression} in the main text.

\begin{corollary}[Asymptotic expression for orthant probability in equicorrelated case]
Let $\rho \in [0,1)$ and $(X_1,...,X_n) \sim \mathcal{N}(\mathbf{0}, \boldsymbol{\Sigma})$ with $\boldsymbol{\Sigma}_{ij} = \rho$ for $i\neq j$ and $ \Sigma_{ii} = 1 $ for all $ i $. Then as $n\rho\rightarrow\infty$,
\begin{align*}
    \P(X_1\geq 0, X_2\geq 0,\dots,X_n\geq 0) \sim \sqrt{\frac{1-\rho}{\rho}} \Gamma\left(\frac{1-\rho}{\rho}\right)(4\pi\log(n))^{\frac{1}{2}(\frac{1-\rho}{\rho}-1)}n^{-\frac{1-\rho}{\rho}}.
\end{align*}
\end{corollary}

\subsection{Lower bound on test error for Gaussian model}
Recall the Gaussian data model:
\begin{align}
(\xb, y) \sim \frac{1}{2}(N_+, 1) + \frac{1}{2}(N_-, -1)
\label{gaussian-data-model}
\end{align}
where $N_+ \sim \mathcal{N}(\mu, \Sigma) $, $ N_- \sim \mathcal{N}(-\mu, \Sigma)$ and $\mu \in \R^d $, $ \Sigma \in \mathcal{S}^{d}_+$. For this model, we have the lower bound 
\begin{align}
    \Er(\wb)\geq \varepsilon^* \geq \Phi(-\sqrt{\mu^\top\Sigma^{-1}\mu})
\end{align}
To see this, define the norm $\|\xb\|_\Sigma^2 = \xb^\top \Sigma \xb$. The value $\varepsilon^*$ satisfies
\begin{align*}
-\Phi^{-1}(\varepsilon^*) &= \text{ess} \sup_{\wb\in \VS}\frac{\wb^\top \mu}{\sqrt{\wb^\top \Sigma \wb}} \leq  \sup_{\wb \in \VS} \frac{\wb^\top \mu}{\sqrt{\wb^\top \Sigma \wb}} = \sup_{\wb\in \VS,\, \|\wb\|_\Sigma = 1} \wb^\top \mu\\ 
&\leq \sup_{\|\wb\|_{\Sigma} = 1} \wb^\top \mu 
=\|\mu\|_{\Sigma^{-1}} =  \sqrt{\mu^\top \Sigma^{-1}\mu}   
\end{align*}
where we use the fact that $\|\cdot\|_{\Sigma^{-1}}$ is the dual norm to $\|\cdot\|_{\Sigma}$. Hence solving for $\varepsilon^*$, we get the lower~bound 
$$
\varepsilon^* \geq \Phi(-\sqrt{\mu^\top \Sigma^{-1}\mu})  .
$$
When $\Sigma = \sigma^2 I$, this lower bound reduces to the usual signal-to-noise ratio $\Phi(-\frac{\|\mu\|}{\sigma})$.

\section{Review of LIN-ESS Algorithm and Additional Empirical Results}

\subsection{Review of LIN-ESS}
In this section, we briefly review the \textsc{lin-ess} algorithm introduced in \cite{Gessner2020IntegralsConstraints}, which is the main computational tool we use in our empirical evaluation. \textsc{lin-ess} builds on the Elliptical Slice Sampling algorithm \cite{Murray2010EllipticalSampling}, which can be used to sample from a posterior under a $\mathcal{N}(\mu, \Sigma)$ prior and generic likelihood $L$. The generic algorithm works as follows: given a starting point $\xb_0$ and a new sample $\xb' \sim \mathcal{N}(\mu,\Sigma)$, construct an ellipse passing through these two points:
$$
\xb(\theta) = \xb_0 \cos(\theta) + \xb'\sin(\theta).
$$
We then sample an angle $\hat{\theta}$ randomly and accept $\xb(\hat{\theta})$ if the likelihood at this point $L(\xb(\hat{\theta}))$ is sufficiently large. Otherwise, we sample a new angle in a narrower band of feasible values. While a general and provably valid\footnote{Meaning, it can be shown to have the true posterior as a unique stationary distribution.} algorithm, this procedure can be slow, as many samples may be rejected before finding an acceptable~sample. 

The key insight of \cite{Gessner2020IntegralsConstraints} is that when the likelihood has the form $L(\xb) = \prod_{i=1}^n \mathbb{1}(\mathbf{a}_i^\top \xb  + b_i \geq 0)$, the region of feasible angles $\theta$ can be obtained analytically, avoiding the need to reject infeasible $\hat{\theta}$. This results in significantly faster computation, even in high dimensions.

\subsection{Additional empirical results}

In Figures~\ref{fig:example_APP_FIRST_EX}-\ref{fig:example_APP_FOURTH_EX} we provide additional empirical results, complementing the results presented in the main~text.

\begin{figure}
    \centering
    {{\includegraphics[scale=.5]{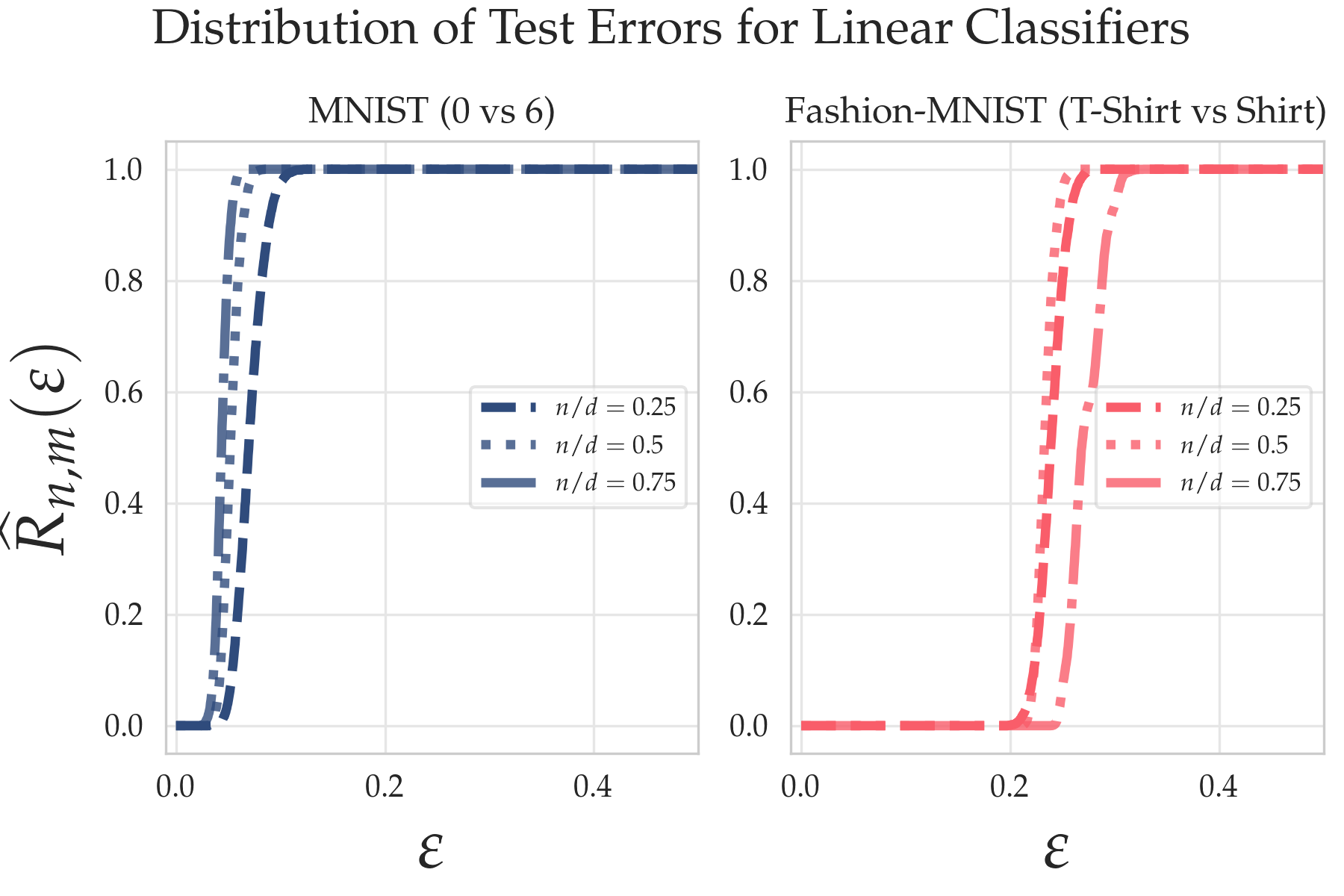} }}%
    \qquad
    {{\includegraphics[scale=.5]{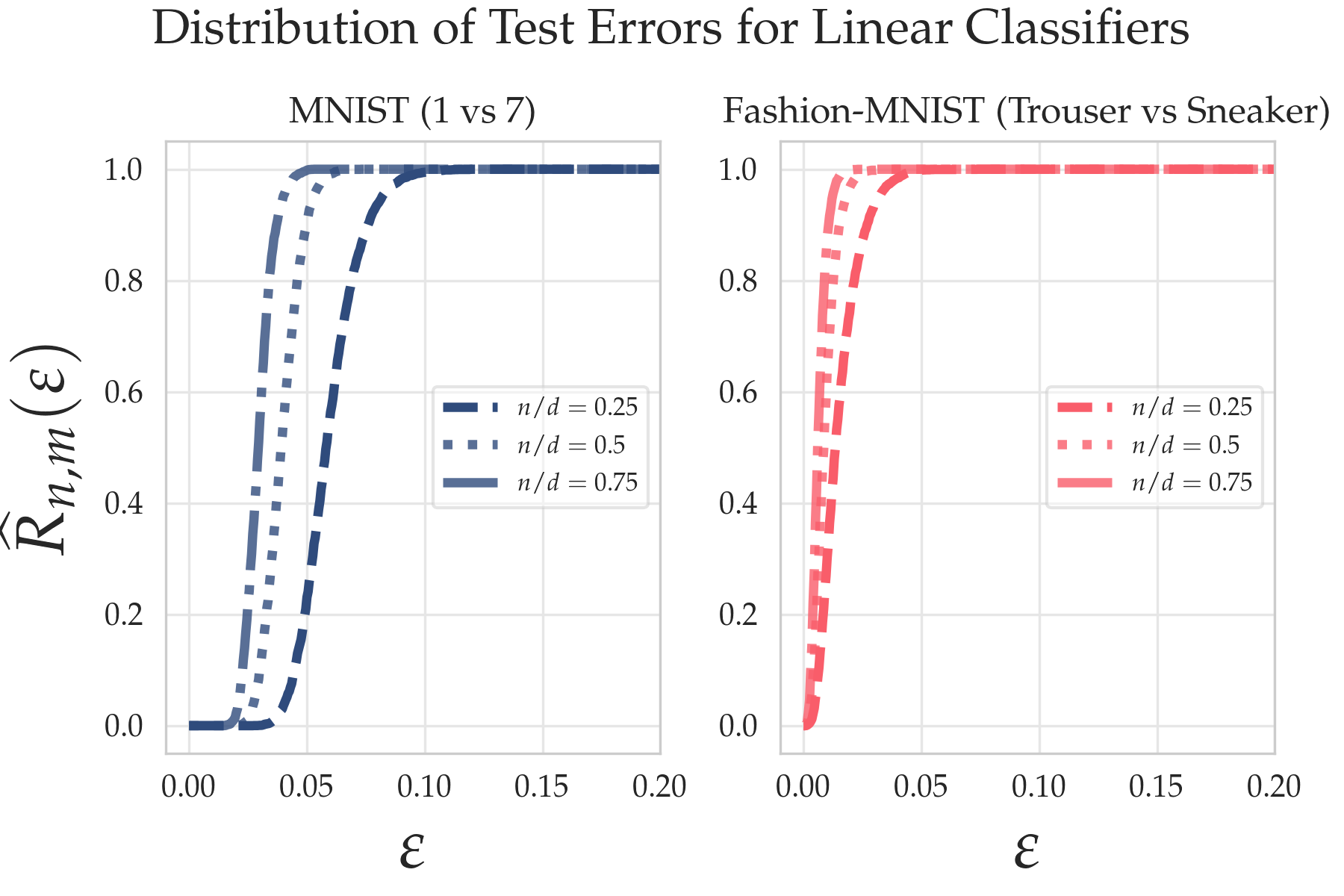} }}%
    \caption{Test error distributions for linear classifiers, similar to Figure 2 in the main text. Here we try four different binary problems; 0 vs 1 classification for \textsc{mnist}, T-shirt vs Shirt classification for \textsc{fashion-mnist}, 1 vs 7 classification for \textsc{mnist} and Trouser vs Sneaker classification for \textsc{fashion-mnist}. While we observe similar qualitative behavior as in the main text, we do note substantial differences in the location of the distributions based on the difficulty of the task. For example, note that T-shirt vs Shirt classification is significantly more difficult than 0 vs 6 classification, while Trouser vs Sneaker classification is easier that 0 vs 1 classification.}%
    \label{fig:example_APP_FIRST_EX}%
\end{figure}

\begin{figure}
    \centering
    {{\includegraphics[scale=.3]{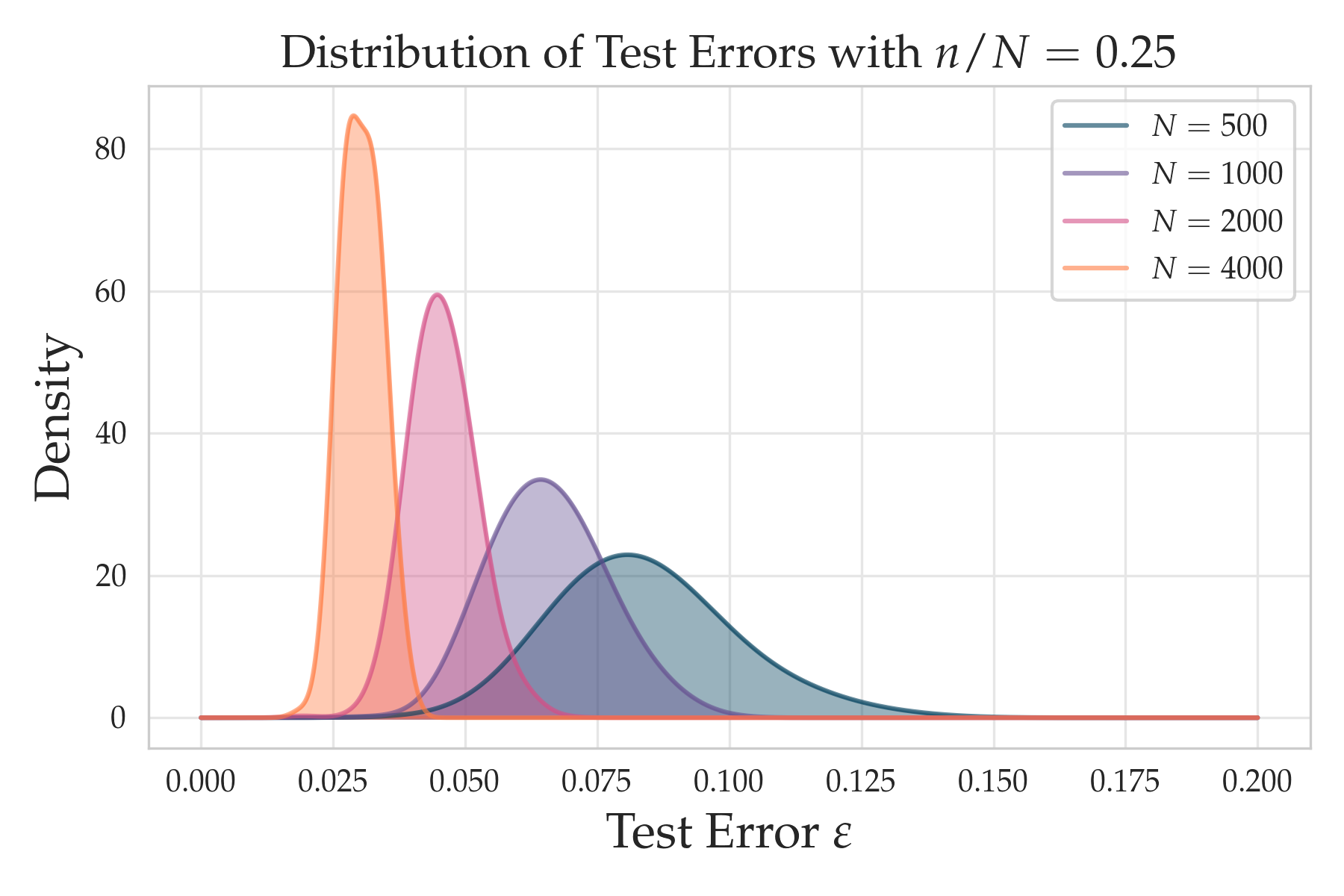} }}%
    \qquad
    {{\includegraphics[scale=.3]{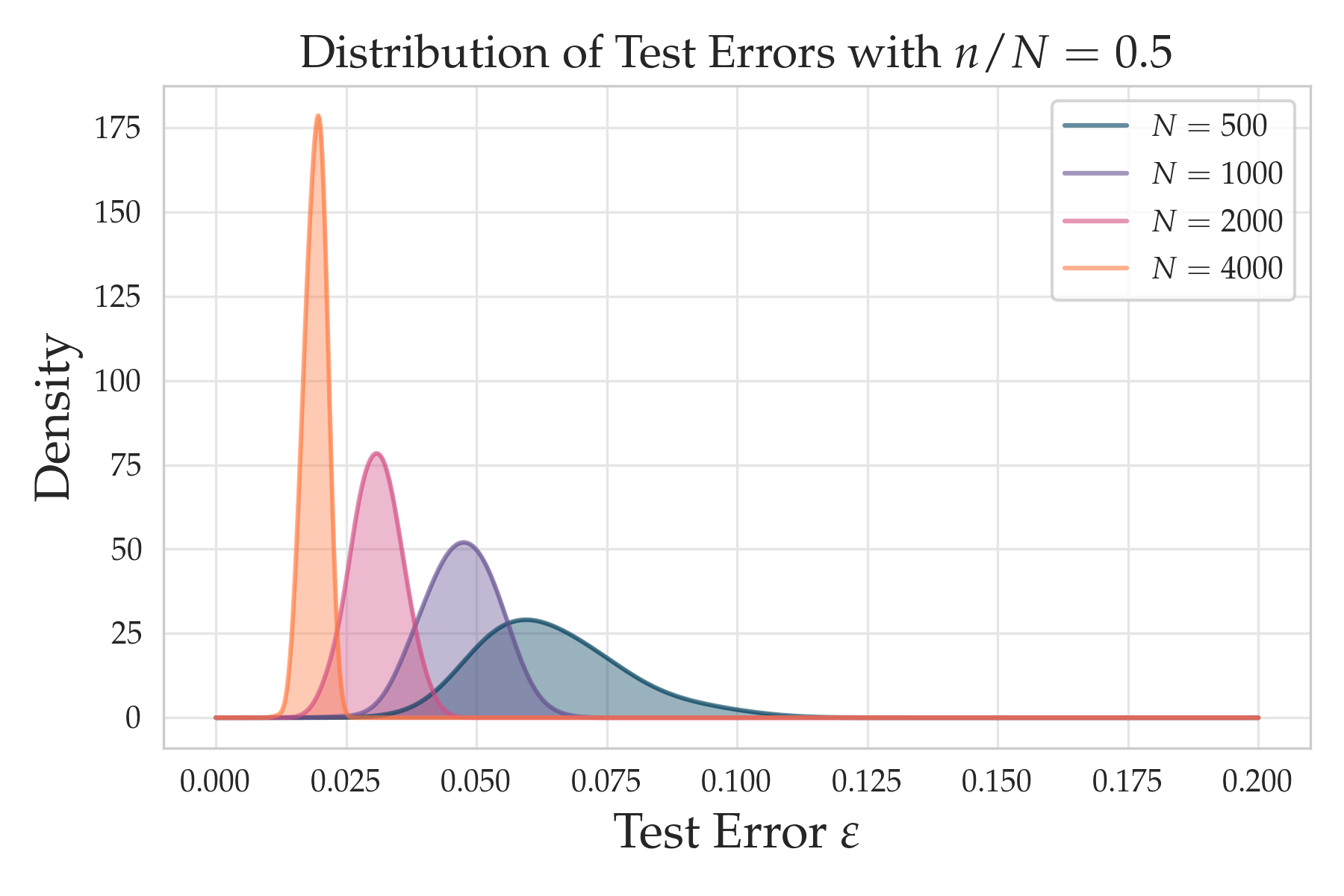}}}%
    \qquad
    {{\includegraphics[scale=.3]{fig/rrf_mnist_pdf_plot_alpha=0.5.png}}}%
    \caption{Test error distributions for random ReLU feature classifiers on \textsc{mnist}. Here we plot PDFs, similar to Figure \ref{fig:pdf_figure}, which we fit using a Gaussian kernel density estimator. These plots are not as precise as the CDF plots shown in the main text, but are more usually for seeing visually the concentration phenomenon as $N\to\infty$.}%
    \label{fig:example_APP_THIRD_EX}%
\end{figure}

\begin{figure}
    \centering
    \includegraphics[scale=.5]{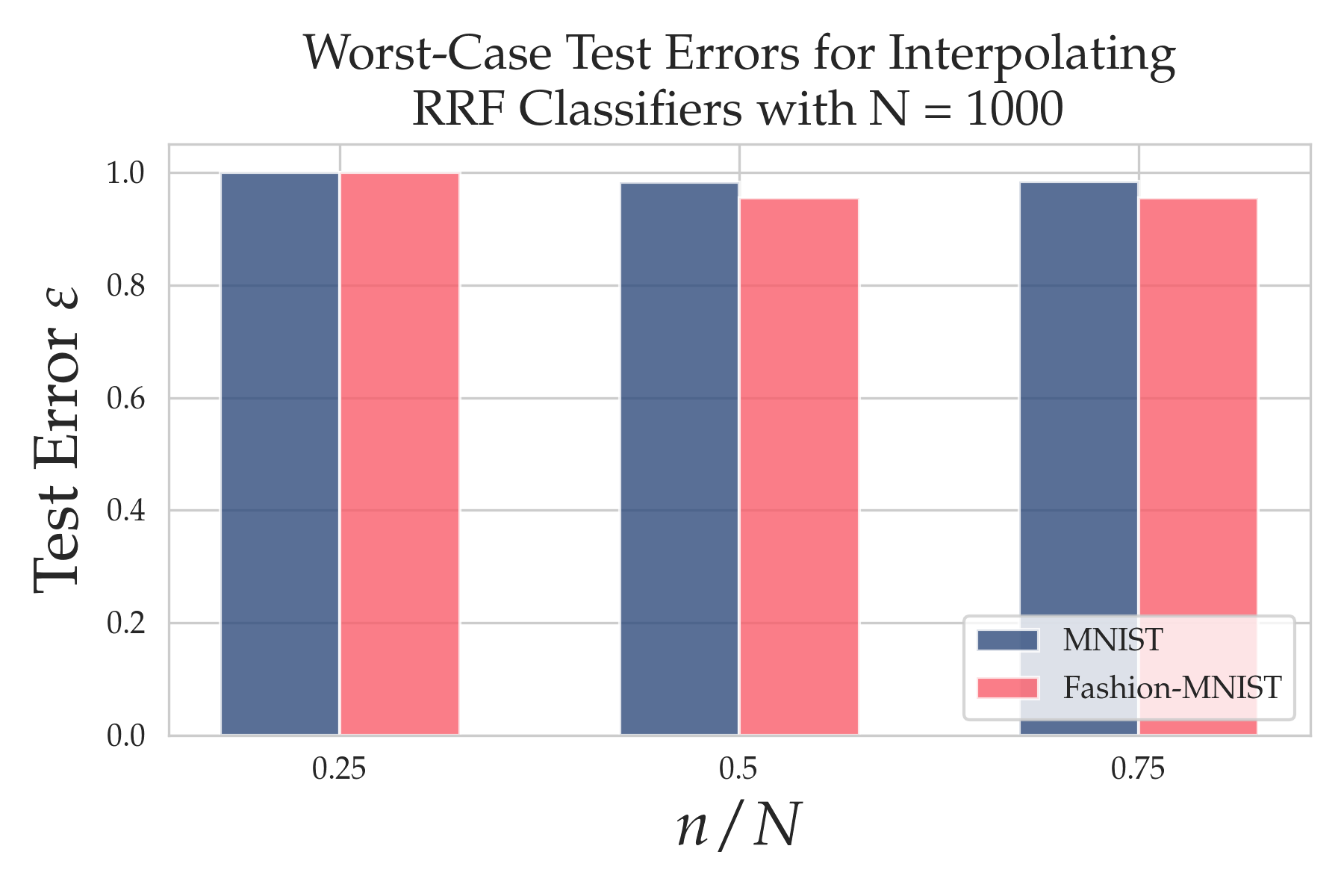}
    \caption{Example of worst case test errors for random ReLU feature models, for $N=1000$ random features. We observe again that bad classifiers do indeed exist, despite an abundance of classifiers with low test error.}%
    \label{fig:example_APP_SECOND_EX}%
\end{figure}

\begin{figure}
    \centering
    \includegraphics[scale=.5]{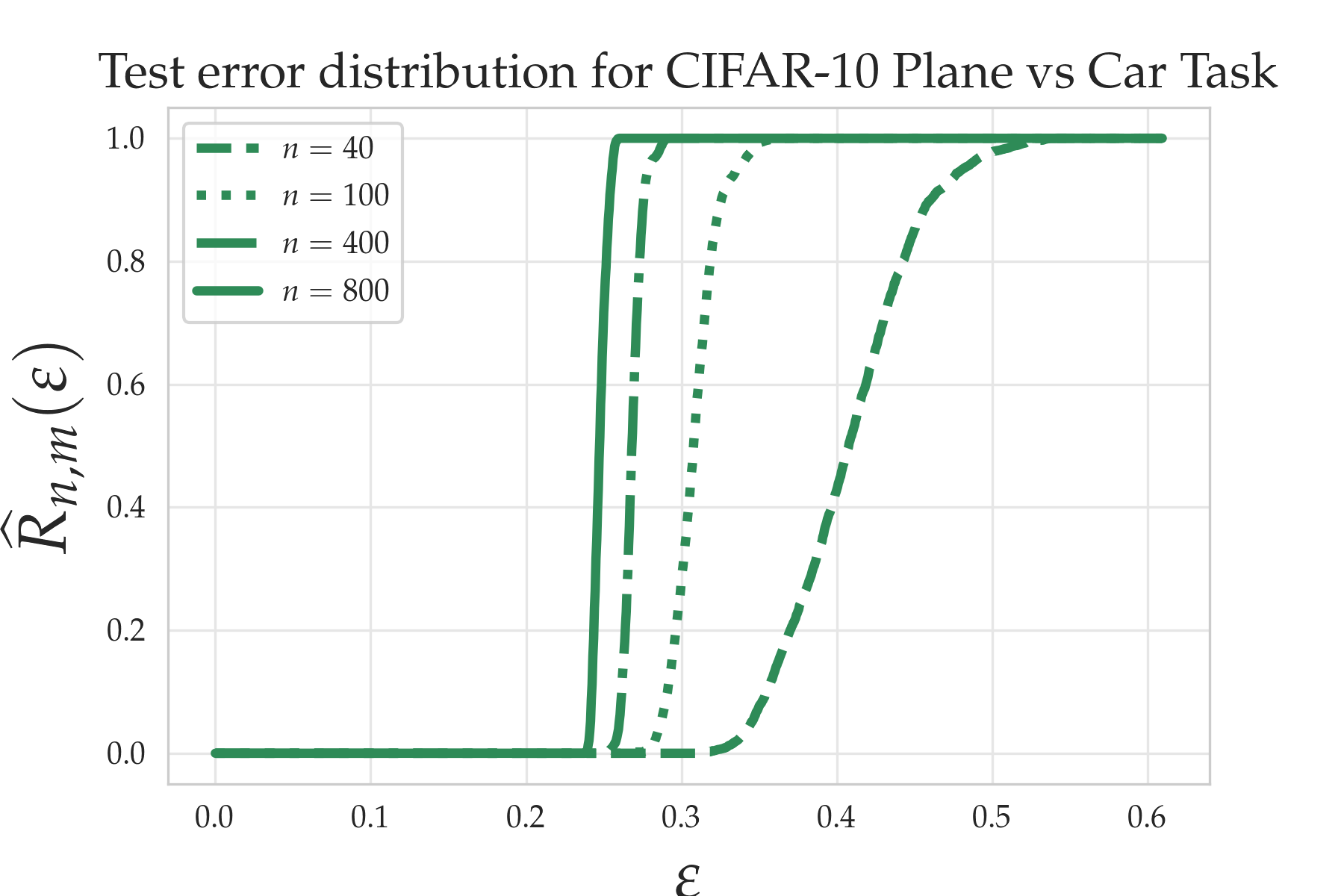}
    \caption{As an additional example, here we plot the test error distribution for linear classification on the plane vs car task on the \textsc{cifar10} dataset.}%
    \label{fig:example_APP_FOURTH_EX}%
\end{figure}
\end{document}